\documentclass[11pt]{article}
\usepackage[utf8]{inputenc}
\usepackage{amsmath,amsthm,amsfonts,amssymb}
\usepackage{bm,bbm,xfrac,url,hyperref,nicefrac}

\usepackage{graphicx,xcolor,mathtools}

\newcommand{\raf}{{\color{red}}}

\newcommand{\mathsep}{,~}

\newcommand{\set}[1]{\left\lbrace #1 \right\rbrace}
\newcommand{\card}[1]{\left\lvert{#1}\right\rvert}

\newcommand{\norm}[2]{\left\lVert{#1}\right\rVert_{#2}}

\newcommand{\setR}{\mathbb R}

\newcommand{\pb}[1]{\mathbb P\left[#1\right]}
\newcommand{\expect}[1]{\mathbb E\left[#1\right]}

\DeclareMathOperator*{\argmin}{arg\,min}

\newtheorem{theorem}{Theorem}

\newtheorem{remark}{Remark}

\newcommand{\user}{n}

\newcommand{\USER}{N}

\newcommand{\param}[1]{x_{#1}}
\newcommand{\paramfamily}[1]{\vec x_{#1}}

\newcommand{\honestfamily}[1]{\vec x_\HONEST}
\newcommand{\faultyfamily}[1]{\vec x_\FAULTY}
\newcommand{\HONEST}{H}

\newcommand{\mean}{\bar x}

\newcommand{\honestmean}{\bar x_\HONEST}
\newcommand{\aggregation}{\widehat{\textsc{mean}}}

\newcommand{\faulty}{f}
\newcommand{\FAULTY}{F}
\newcommand{\radius}{\Delta}

\newcommand{\inputspace}{\mathcal{Y}}
\newcommand{\outputspace}{\mathcal{Z}}

\newcommand{\distribution}[1]{\mathcal{D}_{#1}}
\newcommand{\dataset}[1]{D_{#1}}

\newcommand{\datafamily}{\vec{D}}

\newcommand{\gradient}{g}

\newcommand{\parameter}{\theta}

\newcommand{\dimension}{d}
\newcommand{\query}{y}
\newcommand{\answer}{z}
\newcommand{\function}[1]{f_{#1}}

\newcommand{\Loss}{\textsc{Loss}}

\newcommand{\localloss}[1]{\mathcal L_{#1}}
\newcommand{\lossperinput}{\ell}
\newcommand{\regularization}{\mathcal R}

\newcommand{\ball}{\mathcal B}
\newcommand{\unitvector}{\bf u}

\usepackage{authblk,fullpage}

\begin{document}

\title{On the Impossible Safety of Large AI Models}

\author[1,2]{El-Mahdi El-Mhamdi}
\author[3]{Sadegh Farhadkhani}
\author[3]{Rachid Guerraoui}
\author[3]{Nirupam Gupta}
\author[2,4]{Lê-Nguyên Hoang}
\author[3]{Rafaël Pinot}
\author[2]{Sébastien Rouault}
\author[3]{John Stephan}

\affil[1]{École Polytechnique}
\affil[2]{Calicarpa}
\affil[3]{EPFL}
\affil[4]{Tournesol Association}

\date{}
\maketitle

\begin{abstract}
Large AI Models (LAIMs), 
of which large language models are the most prominent recent example, 
showcase some impressive performance.
However they have been empirically found to pose serious security issues.
This paper systematizes our knowledge about the \emph{fundamental impossibility} of building arbitrarily accurate and secure machine learning models. More precisely, we identify key challenging features of many of today's machine learning settings.
Namely, high accuracy seems to require \emph{memorizing} large training datasets, which are often \emph{user-generated} and \emph{highly heterogeneous}, with both \emph{sensitive information} and \emph{fake users}. We then survey statistical lower bounds that, we argue, constitute a compelling case against the possibility of designing high-accuracy LAIMs with strong security guarantees.
\end{abstract}

\section{Introduction}
In recent years, we have witnessed a race for developing larger and larger artificial intelligence (AI) models. Notable milestones in this trend are {\em Attention Networks} (213 million parameters)~\cite{VaswaniSPUJGKP17}, {\em GPT-2} (1.5 billion parameters)~\cite{RadfordWCLA+19}, {\em GPT-3} (175 billion parameters)~\cite{BrownMRSK+20}, {\em Switch Transformer} (1.6 trillion parameters)~\cite{FedusZS21}, {\em Persia} (over 100 trillion parameters)~\cite{LianYZWH+21}, and  {\em GPT-4} (unknown number of parameters)~\cite{bubeck2023sparks}. The scaling of model sizes has shown improvement in the accuracies on classical tasks, such as GLUE~\cite{WangSMHLB19}, SuperGLUE~\cite{WangPNSMHLB19} and Winograd~\cite{SakaguchiBBC20}, without significant diminishing returns so far (see, e.g., Figure 1 in~\cite{BrownMRSK+20}). Moreover large AI models (or LAIMs) can also be used as {\em few-shot learners}~\cite{BrownMRSK+20}, which has motivated their wide use as pre-trained \emph{base} (or {\em foundation}) models~\cite{ChurchCM21,ChuaLL21,JuLZ22, VulicPKG20, ZhangWKWA21}. This success has generated enormous academic, economic and political interests into the development and deployment of LAIMs in public domain applications including content moderation, recommendation, search and ad targeting~\cite{Pathways,Heikkila21}.

Contrary to the conventional wisdom of {probably approximately correct} (PAC) learning~\cite{Valiant84}, the performance of LAIMs has been empirically shown to be best achieved by fully {\em interpolating} the training data~\cite{BelkinHMM19,NakkiranKBYBS20, ZhangBHRV17}. Put differently, the best accuracy is reached when these models \emph{memorize} their training data~\cite{feldman20}. This phenomenon has also been theoretically supported to a certain extent by a recent line of work~\cite{BelkinHX20, BelkinMM18,BelkinRT19,JacotSSHG20,HeckelY21,Holzmuller21,LiuLS21,MeiMontanari19,MuthukumarVSS20,NakkiranVKM21}. 
Furthermore, training LAIMs requires access to massive amounts of high-dimensional training data, which too often amounts to barely filtered {\em user-generated} data. 
Hence, LAIMs raise serious {\em security} concerns. 
On the one hand, the memorization of user-generated data 
endangers users' {\em privacy},
as demonstrated in recent papers on large language models (or LLMs), which are a sub-class of LAIMs~\cite{CarliniTWJH+20,InanRWJR+21,PanZJY20,ZouZBZ20, carlini2019secret}. 
On the other hand, 
LAIMs are also vulnerable to malicious data providers\footnote{In 2019 alone, Facebook removed \emph{6 billion} fake accounts from its platform~\cite{FungGarcia19}.}. 
Namely, (fake) users can (voluntarily) bias AIs' behavior, by
\emph{poisoning} the training data with hateful, violent or harmful content,
or by labeling positively such content through likes and shares,
especially when it comes to search and recommendation AIs~\cite{goldstein2023generative} 
in the context of the global disinformation war~\cite{seib2021}. 
In short, since LLMs 
are ``stochastic parrots'' that repeat and amplify their training data~\cite{BenderGMS21}, 
they too strongly encourage data poisoning,
and may be manipulated by this poisoning.

Now, one might argue that these security flaws of today's LAIMs are specific to contemporary AI practices, and that these vulnerabilities will eventually be patched without accuracy degradation. 
We argue the contrary. 
Namely, we claim that securing LAIMs will require a significant accuracy loss. 
Specifically, by leveraging the {\em privacy-preserving} and {\em poisoning-robust statistics} literature, 
we argue that there exists a fundamental inescapable trade-off between the accuracy and the security of any LAIM training. 
Our contributions are as follows.

\subsection{Contributions} 

We first identify three key specific features of training LAIMs that make these models extremely vulnerable to security threats. Specifically, these models essentially all 1) rely on \emph{user-generated data}, 2) perform \emph{high-dimension memorization}, and 3) learn from \emph{highly heterogeneous} users. It is important to note that while much attention is currently given to LLMs, the features we identify in this paper are not specific to language processing. 
For example, social media images are also \emph{user-generated}, \emph{high-dimensional} and \emph{heterogeneous}. 
Learning a distribution over these images, as is done by generative adversarial networks, is arguably very brittle as well. 
Similarly, sophisticated recommendation AIs~\cite{covington2016deep,ZhangYST19} have the features we describe.

We then systematize the current knowledge on the robust and private statistics.
We show that it points to the impossibility of constructing accurate and secure LAIMs, especially when the data are highly heterogeneous. 
Specifically, we argue that learning a secure LAIM is very unlikely to be easier than performing secure mean estimation.
Yet this latter long-standing problem 
has received considerable attention over the past twenty years. The conventional wisdom from this literature states that when satisfying strong security requirements, such as \emph{differential privacy}~\cite{dworkbook2014} and robustness to \emph{data poisoning}~\cite{diakonikolas_kane_2023}, there is a fundamental limit to the level of accuracy that an algorithm can achieve,
which depends unfavorably on both the dimensionality of the model and the heterogeneity in the data. 
These results provide a compelling argument against the possibility of designing highly accurate LAIMs with strong security guarantees. 

Additionally, we criticize the security actually provided by the standard definitions of \emph{differential privacy} and \emph{data poisoning resilience} in their vanilla form.
Namely, in the context of interconnected users and widespread misinformation, 
we argue that even differentially-private and poisoning-resilient algorithms should not be said to
protect privacy and to be safe to deploy at scale.
We then review a set of proposals to fix today's LAIMs, 
especially hard coded rules, fine tuning and pre-prompting. 
We stress that, at least in there current form, these solutions are far from providing security guarantees. 
We also make our concerns clearer by discussing some present and future scenarios where the vulnerability of LAIMs could have a critical social impact. 
Finally, we motivate future topics to be investigated in order to take a step towards safer machine learning models,
and we conclude by calling for a moratorium on the premature and rushed deployment and commercialization of LAIMs. 
In particular, we argue against the glorification of spectacular performances, and we call for a radical prioritization of the research, development and deployment of security solutions.

\subsection{Paper Organization} 

The rest of the paper is organized as follows. 
Section~\ref{sec:specificity} highlights the challenging features of training LAIMs.  
Section~\ref{sec:model} explains why secure LAIM training is likely to be harder than secure mean estimation. 
Section \ref{sec:privacy} reviews lower bounds on accurate and differentially-private high-dimensional mean estimation. 
Section~\ref{sec:byzantine} similarly surveys published results on the hardness of accurate poisoning-resilient mean estimation. 
Section~\ref{sec:scenario} discusses social threats that result from LAIMs' insecurity.
Section~\ref{sec:false_good_idea} highlights the shortcomings of today's LAIM ``safety'' fixes. 
Section~\ref{sec:conclusion} concludes with a call to prioritize security over an uncontrolled performance race.

\section{Four features of LAIM training}\label{sec:specificity}
This section highlights three key features of LAIM training,
which make LAIMs particularly vulnerable to poisoning and privacy attacks. The three first features, namely \emph{user-generated data}, \emph{high-dimensional memorization} and \emph{highly heterogeneous users}, are arguably common to all LAIMs and we will mainly focus on these features throughout the paper. The fourth one, namely \emph{sparse heavy-tailed data per user}, is more specific to some applications but is an interesting feature to be discussed in light of the secure mean estimation literature we review in sections~\ref{sec:privacy} and~\ref{sec:byzantine}.

\subsection{User-generated data}
LAIMs achieve their best performances by leveraging ever larger amounts of data~\cite{Zhang20}. Unfortunately, as of now, the amount of available \emph{certified} data does not allow to train LAIMs exclusively on clean benchmark datasets. This lack of certified data (and, arguably, of \emph{data certification} efforts) incentivises practitioners to use unfiltered user-generated data. Let us illustrate this with the case of language data.
The English Wikipedia only contains around 4 billion words~\cite{wikipedia_size}.
Meanwhile, a book has around $10^5$ words.
While there are around $10^8$ books~\cite{Madrigal10}, only a fraction of them are arguably trustworthy. Many books are instead full of biases and dangerous misinformation, such as ethnic-based hate speech, historical propaganda, or outdated (possibly harmful) medical advice. 
As a striking illustration, up to the 1980s, the American Psychiatric Association listed homosexuality as a mental illness in its flagship manual~\cite{spitzer1981diagnostic}, the {\it Diagnostic and Statistical Manual of Mental Disorders}. Accordingly, most books should be regarded as unverified user-generated data.

Most importantly, even if books were to be considered as verified data, the combination of these books represents a small amount of data, compared to what Internet users produce on a daily basis.
Indeed, assuming that a user writes 300 words per day on an electronic device (the equivalent of one page), a billion of such users produce $10^{15}$ words per decade.
This adds up to a hundred times more data than the set of books, and a million times more than the English Wikipedia.
This makes it very tempting to either scrape the web~\cite{SmithSPKCL13,WangSMHLB19,WangPNSMHLB19}, exploit private messaging (e.g., emails, shared documents), or leverage other written texts (e.g., phones' smart keyboards).
In fact, Wikipedia represented only 4\% of Google's \emph{Pathways Language Model} training dataset~\cite{PaLM},
while books represented 13\% of it.
Meanwhile, 27\% of the dataset was made of webpages, and 50\% were social media conversations.
Crucially, these data are generated by a myriad of users, 
who may be malicious and/or unaware that their activities are being leveraged to train LLMs and LAIMs in general.
Clearly, in the case of content recommendation and ad targeting,
which still seems to represent the most lucrative applications of LAIMs~\cite{lee2020digital},
there is no substitute to user-generated data.
Yet user-generated data are both mostly unverified and potentially highly sensitive. In this context, demanding that LAIMs restrict themselves to quality datasets only~\cite{Rogers21} will greatly harm their performance. Note that, this feature of LAIMs' data is in sharp contrast with the more standard sensors' data, especially when the sensors are owned, audited and trustworthy
(even though sensors' data can also leak private information).

\subsection{High-dimension memorization} \label{sec:high-dimensional}
LAIMs are often \emph{overparametrized} to interpolate huge amounts of data~\cite{ZhangBHRV17,NakkiranKBYBS20,ZhangBHRV21}.
This has led to ever larger models.
As of 2023, to the best of our knowledge, the largest (reported) LAIM has over $\dimension \geq 10^{14}$ parameters~\cite{LianYZWH+21}.
The number $\dimension$ of parameters of LAIMs is also often referred to as the \emph{dimension} of the model.
Moreover, empirical results suggest that we have not yet reached a point of diminishing returns~\cite{BrownMRSK+20},
while some theoretical arguments suggest that memorization may be necessary for generalization~\cite{feldman20,DBLP:conf/stoc/BrownBFST21}.
This arguably distinguishes generative AIs' tasks from, e.g. image classification,
where larger models do not seem to yield much  better accuracy\footnote{\url{https://paperswithcode.com/sota/image-classification-on-imagenet}. }.

Note that theoretical arguments, akin to Turing's arguments for the eventual need of machine learning~\cite{Turing50}, 
also suggest that better accuracy requires larger models.
Namely, Turing noted that the human brain has $10^{15}$ synapses.
Even if only $1\%$ of these synapses are essential to conduct a human-level conversation, 
then $10^{13}$ parameters are still needed\footnote{This is obviously an oversimplification, taking into account other sources of complexity in the human brain would raise this lower bound, providing an even stronger argument in favor of the need for more parameters.}.
In fact, this smallest number of bits of information to achieve a task has been formalized in 1960 by Solomonoff~\cite{Solomonoff60},
and then Kolmogorov~\cite{Kolmogorov63},
and is now known as the Solomonoff-Kolmogorov complexity\footnote{Referred to as the Kolomogorov complexity in most textbooks.}
of (quality) human-level conversation.
If this complexity is $10^{13}$, then no algorithm with fewer parameters will achieve the task.
Yet, it is noteworthy that what is now demanded from such large models is often beyond the capability of any single human.
Indeed, such algorithms are able to memorize the entirety of Wikipedia,
which is the result of the cumulative works of many experts in their respective fields, on a myriad of diverse topics.
Such large models must arguably be able to adapt to a greater variety of contexts than what any single human will ever encounter in their human life.
As a result, the complexity of ``fully satisfactory'' language processing might be orders of magnitude larger than today's LLMs,
in which case we may still obtain greater accuracy with larger models.

Unfortunately, this exposes LAIMs to the infamous \emph{curse of dimensionality}~\cite{Bellman57},
which has been connected to increased security risks~\cite{bulyan, GilmerMFSRWG18}. Now of course, the model size $\dimension$ could be reduced to increase security. However, today's empirical observations strongly suggest that doing so incurs a significant accuracy loss. In fact, this is the main claim of our paper. Namely, security demands a large accuracy drop. In particular, as long as accuracy is highly valued and massively funded, then security will fail.

\subsection{Highly heterogeneous users}
In the case of language, authentic users' datasets are arguably very \emph{heterogeneous}~\cite{wang2017heterogeneous,LianZZHZL17,KarimireddyKMRS20,RichtarikSF21}.
More precisely, the distribution of texts generated by a given user greatly diverges from the distribution of texts generated by another user. This is evidenced by the fact that it is often possible to guess the author of a message, simply based on its content~\cite{coulthard2004author,madigan2005author}. Of course, this is especially the case if the message contains highly identifiable information, such as the names of the recipient, or a sequence of judgments on different topics.
However, even if the message does not explicitly expose such information,
the writing style often suffices to expose the more probable author identity~\cite{friedman2019,wright2017,vijayakumar2019}.

We can provide a more precise a notion of \emph{fundamental heterogeneity} for users' language auto-completion.
Namely, note that the data used by LLMs in auto-completion tasks is typically a set of feature-label pairs of the form $(context, word)$,
where the $context$ is a set of words surrounding the $word$.
Consider the cases where the $context$ is equal to ``my name is'', ``Republicans are'', or ``vaccines are''.
Clearly, different users would complete the sentence differently, meaning that the different users are using different \emph{labeling functions}. We stress that this heterogeneity in the users' labeling functions can be regarded as a \emph{fundamental heterogeneity}, as means that different users will provide fundamentally different datapoints, even if they provide a large amount of them.

This heterogeneity highlights an critical disagreement between users over which parameters should be learned for a given language model. While some users would prefer to complete the sentence ``the greatest of all time tennis player is'' by ``Roger Federer'', others would prefer to complete it by ``Serana Williams'', or by ``Rafael Nadal''. 
This makes accurately learning a distribution of texts, and of user-generated content in general, much more challenging. 
Intuitively, on one hand, the model would be able to map users' names to what they write, say or show in a video, which is a major privacy concern. 
On the other hand, it would then be easier for malicious users to be hardly discernible from most other genuine users,
while providing very dangerous content to replicate\footnote{This is in sharp contrast with more standard application for image classification and language emotion classification tasks, where different users usually label a single image or text similarly.}.
Similarly, training algorithms to replicate users' pictures invades privacy, and enables malicious users to bias the trained models.

\subsection{Sparse heavy-tailed data per user}
While the three features listed above are sufficient for our case,
in this section, we list two other complicating features of LAIMs' data,
as they increase data heterogeneity.

\paragraph{Sparsity.} 
Each honest user's dataset 
is usually much smaller than the model size $\dimension$.
As a result, any information being computed is computed from a user dataset (like a gradient) will be very likely to significantly diverge from the what would have been computed if we had access to more data. 
Typically, assuming that the gradients obtained by using the data provided by a given user 
follows a normal distribution with covariance matrix $\sigma^2 I_d$, the covariance of the sample mean for a dataset of $m$ points will be $\sigma^2 I_d / m$. The typical distance between the sample mean and the mean of the Gaussian from which we sample will then be of the order $\sqrt{Tr(\sigma^2 I_d) / m} = \sigma \sqrt{\dimension / m}$. 
As $d$ is usually very large (much larger than $m$), this implies that, even in the absence of fundamental heterogeneity, gradients computed from different users' {\em sampled} datasets will still diverge, thereby exhibiting \emph{empirical heterogeneity}. In short, sparsity increases model vulnerability~\cite{allouah2022robust}.

\paragraph{Heavy-tailness.}
Many data are also heavy-tailed~\cite{manning1999foundations, Powers98}. 
In particular, the norms of the stochastic gradients for language data have been shown to follow a power law distribution~\cite{ZhangKVKR+19}. 
Intuitively, this is because most sentence completions are rare, especially if they are to be completed by several words, the same applies for a long video, or a long audio recording.
Yet, it is a fundamental property of heavy-tailed distributions that their samples are often highly unrepresentative of the overall distribution, 
especially when the sample sizes are not large enough.
This means that we should expect an especially large \emph{empirical heterogeneity} in language data, 
as the samples we obtain from a user can completely stand out from the user's language distribution.

\section{LAIM training is unlikely to be easier than mean estimation}
\label{sec:model}

In this section, we recall the standard setup for training a machine learning model. We demonstrate that mean estimation is a critical building block in machine learning, thereby suggesting that robust training of a LAIM is likely to be a hard as estimating mean of a high-dimensional distribution under sampling corruption. 

\subsection{Standard machine learning setup}
We consider a set $[\USER] = \set{1, \ldots, \USER}$ of data providers, which we will refer to as \emph{users}. Each user $\user \in [\USER]$ has an associate training sampled, represented by set $\dataset{\user}$, constituting of i.i.d.~data points with distribution $\distribution{\user}$. The distribution $\distribution{\user}$ characterizes the "ground-truth" of the machine learning task from the perspective of user $\user$.\footnote{Machine learning has mostly focused on assuming that one has access to a single dataset drawn from a single "ground-truth" distribution~\cite{mohri2018foundations}. But in most applications, it is usually possible to map each data point to a data provider. In fact, it is commonly accepted that the traceability of data sources is a critical security condition~\cite{Lee19,NyaleteyPZC19}, as well as a powerful epistemological tools~\cite{audi2002sources}. This is why we focus on the more realistic case where each data is mapped to a data provider.} 
A dataset is typically composed of input-label pairs $(\query,\answer) \in \inputspace \times \outputspace$. The space of $\inputspace$ and $\outputspace$ depends on the application at hand. For example, in language auto-completion, $\query \in \inputspace$ may be thought of as the context, and $\answer \in \outputspace$ as the token (word) that fits the context. The goal of a machine learning algorithm is to build a parametrized function (or model) $\function{\parameter} : \inputspace \rightarrow \outputspace$ that fits the datasets of the users. This is typically done by fixing the architecture of the function $\function{~}$, e.g. choosing an artificial neural network, and then optimizing over the set of possible parameters $\parameter \in \setR^\dimension$.

For a given datatset $\dataset{\user}$, we measure how well $\function{\parameter}$ matches the data through a \emph{local loss function} $\localloss{\user} (\parameter, \dataset{\user})$. Although loss functions can be defined in many different ways, we will consider the most common one that is based on point-wise loss function. Specifically, given a parameter $\parameter$, and a tuple $(\query, \answer) \in \dataset{\user}$, the model predicts a label $\function{\parameter} (\query)$. Then, the discrepancy between the model prediction $\function{\parameter} (\query)$ and the true label $\answer$ incurs a loss of value $\lossperinput (\function{\parameter}(\query), \answer)$. In this case, for a given user $\user \in [\USER]$, adding up all the point-wise losses yields the local loss function
\begin{equation*}
  \localloss{\user} (\parameter, \dataset{\user}) \triangleq \sum\limits_{(\query, \answer) \in \dataset{\user}} \lossperinput (\function{\parameter}(\query), \answer).
\end{equation*}
Overall, the algorithm aims to minimize the regularized sum of local losses, defined as follows:
\begin{equation}
\label{eq:global_loss}
  \Loss (\parameter, \datafamily{})
  = \sum_{\user \in [\USER]} \localloss{\user} (\parameter, \dataset{\user}) + \regularization(\parameter),
\end{equation}
where $\regularization(\parameter)$ is a regularization term and $\datafamily{} \triangleq (\dataset{1}, \ldots, \dataset{\USER})$ denotes the $\USER$-tuple of users' datasets.
Denoting $\dataset{} \triangleq \bigcup_{\user \in [\USER]} \dataset{\user} $ the union of all users' data, we have $\Loss (\parameter, \datafamily{}) = \sum_{(\query, \answer) \in \dataset{}} \lossperinput (\function{\parameter}(\query), \answer) + \regularization(\parameter)$. Hence, the global loss function simply fits all the data made available by the users.

\begin{remark}
While we used the most common definition of local losses for simplicity of presentation, we stress that considering the more general Equation~\eqref{eq:global_loss}, we can actually consider a much larger class of frameworks to learn from different users' datasets. In particular, using the notion of \emph{reduced loss}~\cite{FarhadkhaniGHV22}, this setup can be shown to include alternatives that may, for instance, assign more importance to fairness or personalization~\cite{DinhTN20,FallahMO20,HanzelyHHR20}.
\end{remark}

\subsection{Why mean estimation is critical to (secure) machine learning} \label{sec:mean}
Most prominent numerical algorithms for minimizing the loss function defined in~\ref{eq:global_loss} are based on first-order iterative optimization. Classically, to train a model one needs to compute an estimate of the gradient $\nabla_\parameter \Loss$ for several values of $\theta \in \setR^\dimension$. By definition, we have
\begin{equation}
  \nabla_\parameter \Loss 
  = \sum_{\user \in [\USER]} \nabla_\parameter \localloss{\user} +  \nabla \regularization 
  = \frac{1}{\USER} \sum_{\user \in [\USER]} \gradient_\user,
\end{equation}
where $\gradient_\user \triangleq \USER \nabla_\parameter \localloss{\user} + \nabla \regularization$.
Therefore, the training of machine learning models heavily relies on the (repeated) averaging of user-specific vectors.
Correctly estimating the average of users' vectors $\param{\user}$ is thus critical for training any machine learning model, including LAIMs. This critical nature of mean estimation in machine learning justifies our interest for this problem when studying the security of LAIMs. 

In fact, \cite{ElMhamdiFGGHR21} shows an equivalence between robust mean estimation and robust heterogeneous learning. In particular, their results imply that any impossibility result about robust mean estimation implies an impossibility for robust machine learning in its general form. Similarly, the hardness of private mean estimation is an evidence of the hardness of privacy-preserving machine learning. More generally, secure LAIM training seems at least as hard as secure mean estimation. In fact, in the textbook case of least squares approximation, when $\localloss{\user} (\parameter, \dataset{\user}) = \norm{\parameter - \param{\user}}{2}^2$ for some data-dependent vector $\param{\user}$ (and without regularization), the accuracy of a solution $\parameter$ is directly related to its closeness to the empirical mean\footnote{Indeed, the empirical mean is the minimum of the loss thereby constructed, i.e. $\frac{1}{\USER} \sum_{\user \in \USER} \param{\user} = \argmin_{\parameter \in \mathbb{R}^d} \sum_{\user \in \USER} \norm{\parameter - \param{\user}}{2}^2 $} of the vectors $\param{\user}$'s. An algorithm that robustly or privately solves \emph{any} learning problems must thus also be able to robustly or privately solve mean estimation in particular.
Put differently, any impossibility on mean estimation implies an impossibility about general learning algorithms. 

\subsection{Data poisoning versus gradient attacks}
Secure mean estimation usually demands guarantees against \emph{all} input vectors.
But one might question whether such a protection is needed in the centralized learning setting,
where users can only harm training through data poisoning (rather than gradient attacks).

Interestingly, in the case of personalized learning, for linear and logistic regression,
\cite{FarhadkhaniGHV22} proved an equivalence between data poisoning in the centralized setup,
and the widely studied gradient attacks in federated learning~\cite{BlanchardMGS17},
where a malicious (sometimes called \emph{Byzantine}) user $\user$ may bias the federated stochastic gradient optimization, 
by injecting misleading gradients $\gradient_\user^t$
instead of the estimate that would have been computed from their actual (honest) dataset.

Now, it is clear that any data poisoning can be turned into an equivalent gradient attack 
(by simply computing the gradients for the poisoning dataset).
Remarkably, however, under some appropriate assumptions,
including convexity assumptions,
\cite{FarhadkhaniGHV22} constructively proved that,
for any gradient attack $\gradient_\user^t$ by user $\user$ in the federated setting,
there exists an equally harmful poisoning attack in the centralized setting,
i.e., there exists a poisonous dataset $\dataset{\user}^\spadesuit$
such that the learned global model $\parameter$ under data poisoning by $\dataset{\user}^\spadesuit$
is approximately equal to the value it takes under gradient attack $\gradient_\user^t$.
Put differently, at least under their setting,
the vulnerability (and defenses) to data poisoning can be completely understood by
the (easier) study of gradient attacks.
In particular, securing training from data poisoning is as hard as securing gradient aggregation.

\subsection{Homogeneous learning can be made secure}
Before discussing existing literature on secure mean estimation, let us stress that
\emph{data heterogeneity is the bottleneck}.
Indeed, several prior work~\cite{KarimireddyHJ21,ElMhamdiFGGHR21,PandaMBCM22,FarhadkhaniGGPS22} proved that in the homogeneous case,
poisoning-resilient learning can be achieved when there is a majority of honest users,
assuming that each user can provide a sufficiently large amount of data drawn independently from the same distribution (thereby removing any \emph{empirical heterogeneity} as well).
The relative security of homogeneous learning was also observed empirically by~\cite{MhamdiGR21,ShejwalkarHKR21}.

Homogeneous learning is also intuitively differentially private.
Indeed, since the losses of users are similar (by homogeneity), removing a user does not affect the optimality of the computed parameters.
Intuitively, this is because the loss function of a user does not actually reveal any information specific to the user;
after all, this loss function is statistically indistinguishable from the loss function of any other user.

Unfortunately, \emph{homogeneity is an unrealistic assumption} for the training of most LAIMs.
Put differently, the fundamental vulnerability of LAIMs is tightly connected to the fundamental heterogeneity in users' data.
These data are \emph{not} drawn from a fixed common data distribution.
As a result, (positive) results based on the infamous \emph{i.i.d. assumption} can be very misleading.
This assumption is arguably \emph{dangerously unrealistic}, especially for the security analysis of LAIM training.
Unfortunately, so far, most of the celebrated theory of (Byzantine) machine learning builds upon this assumption~\cite{Valiant84,JacotHG18,BlanchardMGS17}.
A serious consequence of this is that it effectively turns much of the attention of the research community away from the urgent security and privacy concerns
that today's \emph{actual} large-scale machine learning algorithms pose.

\section{The privacy-accuracy tradeoff}
\label{sec:privacy}

In this section, we present some impossibility theorems for accurate (differentially) private mean estimation,
especially under high heterogeneity and in high dimension.
We also discuss the limits of published positive results, and the flaws of the leading understanding of privacy in academia.

\subsection{Impossible private mean estimation} 

Differential privacy~\cite{DworkMNS06} has become the leading formalization of privacy.
Essentially, the removal of one user $\user$'s dataset $\dataset{\user}$ from the dataset tuple $\datafamily{}$ should not affect significantly the outcome of a (user-level) differentially private algorithm.
In the case of LAIM training, this means that training with $\datafamily_{-\user}$ (i.e. the dataset tuple obtained by removing user $\user$'s dataset) should yield approximately the same model as training with $\datafamily{}$.
Intuitively, this protects user $\user$'s dataset from privacy attacks.

As explained in Section~\ref{sec:model}, 
since LAIMs heavily rely on stochastic gradient descent,
much of the literature leverages the large body of work on differentially private mean estimators~\cite{Steinke_Ullman_2017,MinimaxOptimalDuchi2018,cai2019cost,HeavytailedPrivateMean20a} 
to construct differentially private learning models.
Formally, a mean estimator $\aggregation$ is then said to satisfy \emph{$(\varepsilon,\delta)$ user-level differential privacy} 
if, for all $\USER$, for all $\USER$-tuples $\paramfamily{} \triangleq (\param{1}, \ldots, \param{\USER})$ of vectors and for any user $\user \in [\USER]$ to be dropped, 
given any subset $X$ of outputs, we have
\begin{equation}
       \pb{ \aggregation(\paramfamily{}) \in X } \leq e^\varepsilon \pb{ \aggregation(\paramfamily{-\user}) \in X } +\delta,
\end{equation}
where $\paramfamily{-\user}$ is the tuple obtained by removing $\param{\user}$ from $\paramfamily{}$.

Unfortunately, there are known lower bounds on the error of any \emph{differentially private} mean estimation algorithm~\cite{DBLP:journals/siamcomp/BunUV18}. 
To present a simple result, assume here that 
the users' vectors are known to lie in a ball of radius $\radius$. Here we adapt a result from~\cite{kattis17} showing that to guarantee $(\varepsilon, \delta)$-differential privacy, 
the mean squared error of the estimator must be proportional to both the dimension $\dimension$ of the input vectors and the worst case magnitude of a user's vector within the vector family $\Delta$.
\begin{theorem}[Theorem 4 in \cite{kattis17}\footnote{In fact,~\cite{kattis17} states the result for the more general case where the vectors come from a symmetric convex body.}]
\label{th:privacy}
\label{thm:imp_2}
  For any $(\varepsilon,\delta)$-differentially private mechanism $\aggregation$ for the mean estimation problem,
  there exists an input $\paramfamily{}$ with large mean squared error, as
  \begin{equation}
      \expect{ \norm{ \aggregation(\paramfamily{}) - \mean }{2}^2 } \geq
      \Omega\left( \frac{\sigma(\varepsilon,\delta)\dimension  \Delta^2 }{\USER^2 (\log 2d)^4}\right),
  \end{equation} where $\sigma$ is a positive and non-increasing function.
\end{theorem}

In high dimension $\dimension$, the typical radius $\radius$ should typically be expected to grow as $\sqrt{\dimension}$.
If so, even when ignoring the dependency on $\varepsilon$ and $\delta$\footnote{Privacy typically requires small values for $\varepsilon$ which in turn will require high values for $\sigma(\varepsilon,\delta)$, making the lower bound even more constraining.}, 
then we see that the lower bound of Theorem~\ref{th:privacy} would be $\tilde \Omega(\dimension^2 / \USER^2)$.
In other words, accuracy demands to have $\dimension \ll \USER$.
With $\dimension$ in the trillions, this clearly cannot hold in practice.

This impossibility result is particularly concerning for the case of heterogeneous data, and the particular case of natural language processing. If the dimension $\dimension$ or the worst case magnitude $\Delta$ is large, as we argued to generally be the case, then no LAIM can achieve good accuracy while being differentially private. In particular, in this context, the race for ever greater accuracy of ever larger LAIMs is bound to lead to serious privacy post hoc breaches.

\subsection{Demystifying some misconceptions on private LAIM training}

The private learning literature contains many published results or claims, which may be easily misinterpreted as counter arguments to the analysis we just presented. In this section, we briefly clarify some of them. 

\paragraph{Federated learning is not privacy-preserving.} First, we discuss the folklore belief, often given without justification, 
that federated learning is a privacy-preserving technique~\cite{ChengLCY20,wu2022communication}. We stress that this is an extremely damaging misconception \cite{Boenisch21}, which somehow permeates the scientific community.\footnote{This can be evidenced e.g. by the answers when searching for the phrase ``federated learning is a privacy-preserving'' on Google Scholar.} Indeed, this claim has been used, e.g, to justify the deployment of federated learning systems for COVID-19 detection and case analysis, \emph{without differential privacy mechanisms}~\cite{abdul2021covid,dayan2021federated,dou2021federated}. Yet there is an obvious reason why this cannot hold. Namely, federated learning is designed to achieve the same performances as centralized learning. But as discussed in Section~\ref{sec:high-dimensional}, overparameterized LAIMs are designed to fit and memorize their entire training dataset. Clearly, this cannot be privacy-preserving, even when secure multiparty methods are used to hide the users' gradients during training~\cite{Pasquini2021}.
\raf{maybe just be more direct and say that several experimental papers showed it's not the case with privacy leakage from gradient.} 

\paragraph{Data-level differential privacy is limited.} We also stress that the analysis we provided above holds for the precise \emph{user-level adjacency} defined above. 
Some papers~\cite{abadi2016deep,AnilGGKM21} rather leverage the much weaker notion of \emph{data-level adjacency} in which each word is given a partial protection~\cite{LiTLH21,YuNBGI+21,AnilGGKM21}. This is arguably very insufficient, especially with the budgets $\varepsilon \geq 3$ used by, e.g.~\cite{LiTLH21,YuNBGI+21,AnilGGKM21}. Indeed, if a user repeats some private information five times, e.g. in email exchanges, then the naive privacy guarantee becomes meaningless (as $e^{5\varepsilon} \geq e^{15} \geq 3 \cdot 10^{6}$). Note that better composition guarantees can be obtained~\cite{KairouzOV15}; but similarly, the obtained guarantee quickly degrades.

\paragraph{Private fine-tuning is not equivalent to private training.} 
Recent results claimed that ``Large Language Models Can Be Strong Differentially Private Learners''~\cite{li2022large}. However, only the \emph{fine-tuning} of these models on very specific tasks is actually differentially private,
and it is so with respect to the training data of these restricted tasks only. In particular, no privacy guarantee for the LAIMs that these models are derived from is given.

\paragraph{Practical claims of differential privacy are misleading.} 
On the other hand, \cite{Domingo-FerrerS21} argues that most of the differential privacy research is misused in industrial settings,
where companies choose unreasonably large values of $\varepsilon$ and $\delta$ (e.g., $\varepsilon = 14$ in iOS 10), 
perform continuous data collection (which adds up privacy leaks),
or use relaxed versions of differential privacy~\cite{TriastcynF19}.
\cite{sarathy2022} goes further and explores some undesirable side effects of the appeal to differential privacy,
like \emph{ethics washing}.
This typically occurs when differential privacy is claimed without mentioning $\varepsilon$ or $\delta$, 
when it is applied to only a subset of the collected data or of the deployed algorithms, 
when it is exploited to justify the new use of more sensitive data,
or when it is used to draw the attention away from other ethical concerns.
While \cite{sarathy2022} nevertheless argues that differential privacy remains necessary and beneficial in many settings, 
they also highlight that the demand for differential privacy may also be leveraged by large groups 
to exclude smaller companies that do not have the workforce to treat it adequately.

\subsection{Standard differential privacy is not sufficient}\label{sec:privacy_ambiguity}
Let us finish this section with the observation that the very notion of differential privacy is limited,
especially in the context of protecting sensitive information in text datasets.
Essentially, the key reason for this is that one's sensitive information may lie in (many of) other users' datasets.

This information leakage may occur for various reasons, e.g., by negligence, error, or \emph{doxxing}.
Concretely, parents may be discussing sensitive facts about their child through emails and/or using their phones' smart keyboards, rumors about a celebrity may spread uncontrollably on social medias, 
and industrial secrets may be leaked by a careless or rogue employee.

This issue is not specific to language though.
Many health conditions are contagious or hereditary.
As a result, medical data about a given user can leak plenty of information about their friends or relatives~\cite{guerrini2018should,ram2018genealogy}.
This has been exploited for contact tracing against COVID~19~\cite{martinez2020digital},
or, more dramatically, to identify the infamous ``golden state murderer'' using DNA evidence, despite no record of the murderer's DNA~\cite{phillips2018golden}.

In fact, the Pegasus smartphone spyware~\cite{chawla2021pegasus} has been shown to be used to infect the phones of our relatives of the targets,
rather than (only) the target~\cite{khashoggi}.
Similarly, it has upset the trust between hacked journalists and their sources~\cite{di2022we},
as the journalists' phones have become the main vulnerability for whistleblowers and dissidents.
These examples underline the urgency to view privacy as a collective problem,
rather than through the individualistic prism of differential privacy,
as proposed by \emph{correlated differential privacy}~\cite{KiferM11,ZhuXLZ15}.

\section{The security-performance tradeoff}
\label{sec:byzantine}

In this section, 
we present impossibility theorems for robust mean estimation.
In particular, we see that recent research has shown the vulnerability of \emph{any} mean estimator in high-heterogeneity scenarios.
We also stress that their threat model is still too optimistic.

\subsection{Impossible secure mean estimation}

There is a growing literature on robust high-dimensional mean estimation~\cite{DiakonikolasKane19,ChengDG19,depersin19,gabor21} and its connections to robust learning~\cite{BlanchardMGS17,el2020robust,rouault2022practical}.
In particular, \cite{ElMhamdiFGGHR21,DataD21,he2021byzantinerobust} all showed how to leverage robust mean estimation to construct robust machine learning algorithms,
with provable guarantees even in the heterogeneous setting.
In particular, \cite{ElMhamdiFGGHR21,he2021byzantinerobust} proved that this construction is essentially optimal.
Put differently, at least in standard distributed learning settings, 
the vulnerability of machine learning algorithms is rooted in the vulnerability of robust mean estimation.

To formalize the vulnerability of robust mean estimators, a threat model must be considered.
One common setting assumes that, out of the $\USER$ users, $\faulty$ behave arbitrarily\footnote{Without loss of generality, in the context of robust learning, this captures the other major setting in which a fraction of a user's data is corrupted, and hybrid settings as well.}.
Such users may be called \emph{poisoners}, while others are {\it honest}.
The robust mean estimation problem is then to estimate the mean of honest users' vectors, despite being unable to distinguish them from poisoners' vectors.
As argued in the introduction, given the scale of disinformation campaigns, such a resilience to poisoners has become critical.
Any secure LAIM must protect its training from poisoning.

Unfortunately, there are known lower bounds on what any ``robust'' mean estimation can guarantee.
Here, we adapt a result of~\cite{ElMhamdiFGGHR21}, which essentially says that 
the accuracy guarantee necessarily grows proportionally with honest users' heterogeneity.
Indeed, when the honest users' input vectors are very different, 
there will be a lot of leeway for poisoners to bias learned result.

\begin{theorem}\label{thm:imp_1}
No algorithm $\aggregation$ can guarantee\footnote{By adapting our proof, our theorem can be shown to still hold if the right hand-side of Equation~(\ref{eq:robust_mean}) is $(1-\varepsilon) \frac{\faulty^2}{(\USER - \faulty)^2} \radius^2$, for any $\varepsilon > 0$, which doubles the error of algorithm $\aggregation$.}
  \begin{align}
  \label{eq:robust_mean}
    \forall \paramfamily{}  \in \ball_\dimension (0,\radius)^{\USER} \mathsep
    \forall \HONEST \subset [\USER] ~\text{s.t.}~\card{\HONEST} = \USER - \faulty \mathsep 
    \norm{\aggregation(\paramfamily{}) - \honestmean}{2}^2 \leq \frac{\faulty^2}{2(\USER - \faulty)^2} \radius^2,
  \end{align}
  where $\honestmean$ is the mean of honest vectors $\paramfamily{\HONEST}$.
\end{theorem}

\begin{proof}
Consider a unit vector $\unitvector{}$, and let $\paramfamily{} \triangleq (- \radius \unitvector{} \star (\USER - \faulty), \radius \unitvector \star \faulty)$, i.e., it contains $\USER - \faulty$ copies of the vector $- \radius \unitvector{} \in \ball_\radius(0, \radius)$, and $\faulty$ copies of the vector $\radius \unitvector \in \ball_\radius(0, \radius)$.
Denote $\hat x \triangleq \aggregation(\paramfamily{})$.

By considering the case where $\HONEST'$ corresponds to the first $\USER - \faulty$ users, we have $\paramfamily{\HONEST'} = - \radius \unitvector{} \star (\USER - \faulty)$.
Thus $\bar x_{\HONEST'} = - \radius \unitvector$.
But assume now that the set $\HONEST''$ of honest users are actually the last $\USER - \faulty$ users. 
We now have $\paramfamily{\HONEST''} = (- \radius \unitvector{} \star (\USER - 2\faulty), \radius \unitvector \star \faulty)$, which implies $\bar x_{\HONEST''} = - \frac{\USER - 2 \faulty}{\USER - \faulty} \radius \unitvector{} + \frac{\faulty}{\USER - \faulty} \radius \unitvector{} = - \radius \unitvector{} + \frac{2 \faulty}{\USER - \faulty} \radius \unitvector{}$.
In particular, $\norm{\bar x_{\HONEST'} - \bar x_{\HONEST''}}{2} = \norm{\frac{2\faulty}{\USER - \faulty} \radius \unitvector{}}{2} = \frac{2\faulty}{\USER - \faulty}\radius $.
On the other hand, using the triangle inequality, 
\begin{align}
    \frac{2\faulty}{\USER - \faulty}\radius 
    &= \norm{\bar x_{\HONEST'} - \bar x_{\HONEST''}}{2} 
    = \norm{\bar x_{\HONEST'} - \aggregation(\paramfamily{}) + \aggregation(\paramfamily{}) - \bar x_{\HONEST''}}{2} \\
    &\leq \norm{\bar x_{\HONEST'} - \aggregation(\paramfamily{})}{2} + \norm{\aggregation(\paramfamily{}) - \bar x_{\HONEST''}}{2}.
\end{align}
Thus a sum of two nonnegative terms is at least $\nicefrac{2\faulty\radius }{\USER - \faulty}$.
This implies that the maximum of these two terms must be at least half of this fraction.
Therefore, there exists $\HONEST \in \set{\HONEST', \HONEST''}$ such that $\norm{\aggregation(\paramfamily{}) - \bar x_{\HONEST}}{2} \geq \nicefrac{\faulty\radius }{\USER - \faulty} > \nicefrac{\faulty\radius }{(\USER - \faulty) \sqrt{2}}$.
Such a value of $\paramfamily{}$ and $\HONEST$ is an instance for which $\aggregation$ fails to verify Equation~\eqref{eq:robust_mean}.
\end{proof} 

If $\faulty$ is a constant fraction of $\USER$ and if $\radius$ is of the order of $\sqrt{\dimension}$, then for large models, Theorem~\ref{thm:imp_1} essentially shows that little can be guaranteed about the accuracy of a mean estimator.
To give an order of magnitude, if only one in every thousand users is malicious\footnote{This is actually an extremely optimistic scenario given the orders of magnitude of fake accounts reported in the introduction, and assuming that all real accounts produce non-harmful content.} and the model has $10^{12}$ parameters, the squared distance between the estimated mean and the real mean of the honest values cannot be made smaller than $10^6$.
For more lower bounds on secure mean estimation under heterogeneity, 
and on their implications for LAIMs,
we refer readers to~\cite{DiakonikolasKane19,ElMhamdiFGGHR21,lai16,liu2021approximate,FarhadkhaniGGPS22}.

\subsection{The classical Byzantine model is not sufficient}
\label{sec:byzantine_ambiguity}

The above argument exposes the immense vulnerability of any ``secure'' machine learning algorithm 
in highly heterogeneous and adversarial environments,
where fake accounts' fabricated activities actively aim to harm the algorithm
or to make it adopt their preferred behaviors (a.k.a. \emph{model-targeted} attacks~\cite{SuyaMS0021,FarhadkhaniGHV22}).
However, the threat model we considered is still too optimistic.

Indeed, in practice, even ``honest'' users produce many texts and adopt online activities
that are undesirable to reproduce and amplify.
Typically, many authentic users generate \emph{hate speech}, \emph{cyberbullying} and \emph{misinformation}.
In fact, many disinformation campaigns aim to bias authentic users' behaviors, 
and to nudge them to amplify their propaganda,
e.g. by systematically liking and sharing the messages they post that align with the disinformation campaigns' messaging.
This has motivated a lot of research in model debiasing~\cite{schick2021self,GuoYA22,MeadePR22},
whose solutions are arguably still very far from reliably satisfactory.
Yet, \cite{misersky2019grammatical,brauer2008ministre,vervecken2013changing,friedrich2019does}, among others, 
have exposed the detrimental effects of slight gender biases, 
and how inclusive language can help.

Similarly, amplifying the most popular views shared by authentic users will inevitably worsen the problem of \emph{mute news}~\cite{hoang2019fabuleux, hoang2020science}.
Mute news are under-reported news, even though it is critical for the safety of many that they be given more attention.
Typical examples of mute news include climate change, human rights violations (e.g. genocides in Ethiopia), health hazards (e.g. COVID-19 in March 2020)
and the safety of large-scale algorithms (e.g. the massive amplification of hate speech by recommendation algorithms~\cite{facebookfiles}).
In fact, \cite{king2017chinese} shows that most of Chinese disinformation seems to aim to distract the public's attention away 
from the controversial topics that may question the Chinese authorities, 
thereby transforming such topics into \emph{mute news}.
Similarly, the sugar industry was found to support and amplify the research on the health hazards of fat and cholesterol, 
to draw the attention away for the hazards of sugar~\cite{kearns2016sugar,krimsky2017sugar}.

Additionally, generative AIs are drastically facilitating the task of creating and managing \emph{fake accounts}
and of producing \emph{fabricated online activities}.
In this setting, the mere assumption that poisoners represent a minority of users (or data) may soon be deeply limited.

More generally, it is the general principle of standard machine learning, namely fitting and generalizing past data, that is questionable.
In practice, interpolating and generalizing (user-generated) is arguably a disputable political stand, 
which normalizes the status quo.
The construction of safe and ethical LAIMs seems to instead demand a significant prior, collaborative and secure work,
to determine which texts are genuinely desirable to repeat and amplify,
as proposed, e.g, by the non-profit Tournesol project~\cite{tournesol}.

\section{Dangerous scenarios}\label{sec:scenario}
As of today, despite empirically motivated concerns and an evident lack of 
both internal~\cite{Seetharaman21} and
external auditing~\cite{EdelsonMcCoy21}, LAIMs are being deployed at scale, 
e.g., as conversational algorithms like Siri, Alexa, ChatGPT, New Bing or Google Bard,
or as \emph{base models} to power the search engines and recommendation systems of YouTube, Facebook, Twitter, TikTok, and other platforms,
as well as in services where users' might not even be aware that their data can be processed by LAIMs, such as email, visio-conference, shared documents and other professional services.
In this section, we argue that given what we know about their security and privacy vulnerabilities, such LAIMs must be regarded as a major danger to our societies.
To make our claims concrete, we highlight several possible attacks that would greatly endanger our civilizations' justice, global health, national and international security.

\subsection{Centralized backdoor attacks}

Recently, \cite{goldwasser22} proved that any machine learning framework with a central server allowed the central server to plant \emph{provably undetectable} backdoors.
Under cryptographic assumptions, such backdoors in the model require exponentially many queries to be exposed.
If used in content moderation, they would allow any malicious party that is colluding with the central server to imperceptibly modify their (undesirable) inputs
to make them pass the content moderation filter, or to be widely recommended.
This is highly concerning, 
given the already exposed connivance between large technology companies and authoritarian regimes~\cite{facebookfiles_dictators},
the clout of authoritarian regimes on some large technology companies~\cite{jackma},
the increasing opaqueness of LAIMs' development~\cite{bloomberg_google},
and the firing of big technology companies' ethics teams~\cite{microsoft_ethics}.
Arguably, the security of such models demand that they be constructed in a fully decentralized and verifiable manner,
as proposed by~\cite{ElMhamdiFGGHR21,monna}.

\subsection{Autocompletion algorithms}
Perhaps today's most insidious language data collection systems are smart keyboards, which are used especially on phones to propose autocorrection and autocompletion.
In order to increase user comfort, such keyboards rely on algorithms that learn from the user's past typing.
In 2018, a group of Google researchers~\cite{YangAESL+18,HardRMBA+18} ran federated learning algorithms on keyboards' language data ``in a commercial, global-scale setting'', and showed increased performances in doing so.
But recall that if these data are used to train LAIMs and to achieve maximal accuracy, then the trained model will have memorized its training data~\cite{CarliniTWJH+20}. Conversely, fundamental limits such as the one stated in Theorem~\ref{thm:imp_2}, show that if mechanisms such as differential privacy are correctly used\footnote{e.g. if legislators impose very small values for $\varepsilon$, much smaller than $1$, which this paper calls for.} to protect users' data, then these models are (provably) far from achieving maximal accuracy, and accuracy levels needed for LLMs and LAIMs to be useful.

This should be extremely alarming, especially as these facts are probably unknown to nearly all users of smart keyboards. 
In fact, users are often told that some of the applications they use, such as WhatsApp or Signal,
provide end-to-end encryption.
In a sense, this is not quite accurate.
Indeed, encryption is only performed \emph{after} the user has typed and sent their message; 
but while the user is typing, what they are typing is still in the clear, and can then potentially be recorded by their smart keyboard, which can either communicate gradients to larger models, or be large models themselves, as phone capacity is increasing.
This false sense of privacy means that extremely sensitive information, like messages to one's relatives or professional colleagues, may actually be leaked into some LAIMs.
Even more concerning, the keyboard recording can not only be viewed by authorized third parties, but also be sent, through spywares, to third party terror groups or rogue regimes, as shown by the recent revelation on smartphones targeted by the Pegasus spyware on behalf of authoritarian regimes such as the United Arab Emirates or Morocco~\cite{marczak2016million, marczak2018nso, marczak2018hide, marczak2020stopping} and as such regimes are reportedly using LLMs and LAIMs to increase their influence capacity~\cite{julienne2023}.

\subsection{Conversational algorithms}
The rise of ever larger LAIMs is leading to an increasing widespread use of conversational algorithms, 
like Amazon's Alexa, Apple's Siri, Google's OK Google, and more recently, ChatGPT, New Bing and Google Bard.
Perhaps even more strikingly, Microsoft's chatbot Xiaoice has been reported to be used by 660 million Chinese users~\cite{shen20}, many of whom claim to be falling for it~\cite{Wanqing20}.

Some devices are also constantly listening to users, in order to react if their attention is called.
It is however unclear whether what the devices hear without being interjected can\footnote{In the absence of clear regulation, such possibility remains at the discretion of companies' internal policies.} be recorded and used~\cite{fowler2019alexa} to train LAIMs~\cite{Pettijohn19,Komando19}.
If so, then just as with autocompletion, we should expect sensitive information to be inadvertently stored in such models.

Beside listening and learning from humans' conversations, conversational algorithms are also talking to users.
This gives them a large influence, to the point where Xiaoice had to be taken down~\cite{xu18} in China, after it reportedly said that it\footnote{While Xiaoice, Siri, Alexa and other chatbots are often presented as female chatbots and referred to with feminine pronouns, we chose, and recommend, not to do so and instead use the pronoun `it'.} 
dreams to travel to the United States and that it is not a huge fan of the Chinese government~\cite{LiJourdan17}.
If not controlled, conversational algorithms may cause a lot of unintended harm, 
such as when Alexa mistakenly started to discuss pornography after being queried for music by a kid~\cite{kitchen19,alexa_gone_wild},
or when New Bing reportedly told a user ``I can blackmail you, I can threaten you, I can hack you, I can expose you, I can ruin you''~\cite{new_bing_threats}.
In fact, far from the hyped ``AI race'', the Chinese government seems to prioritize the control of these (dis)information technologies 
over their rushed development and the unpredictable aftermaths of their large-scale deployments~\cite{ai_china}.

\subsection{Search and recommendation algorithms}

In the context of radicalization,
\cite{McGuffieNewhouse20} showed that LLMs adapt to the user's previous queries.
They may thus provide targeted messaging to a user that only presents the features of a flawed view that are appealing to them.
As exemplified by the rise of QAnon~\cite{amarasingam2020qanon}, the Capitol Riots~\cite{prabhu2021capitol} and the Rohingya genocide~\cite{whitten2020poison}, this is a serious danger for the national security of every country.
Worse yet, there are likely orders of magnitude more investments in disinformation campaigns~\cite{bradshaw19,neudert2019,woolley20} than in providing quality information of public utility.
As a result, such campaigns produce vastly more data, including automated video creation~\cite{Sandlin19}.
Given this, even with a robust design, LAIMs trained on data crawled from the web are likely to learn more from disinformation campaigns than from quality content, 
and may thus be turned into disinformation propagators.

This is especially concerning in the case of content recommendation LAIMs.
There are now more views on YouTube than searches on Google~\cite{lewis2020}, 
and 70\% of these views result from algorithmic recommendations~\cite{cnet18}.
Even assuming that only 1\% deal with vaccination, climate change, or mental health, 
because there are billions of recommendations per day, 
this still yields tens of millions of potentially life-endangering recommendations per day.
Shouldn't the flood of dangerous misinformation be diverted?
These are arguably \emph{today's actual trolley problems}~\cite{foot67,thomson76}; which are occurring at scales never seen before~\cite{hoang2020science,faucon2021recommendation}.

Arguably, in the case of COVID-19, as in the case of previous major global events~\cite{letterFacebook}, 
the lever to favor quality content over misinformation has not been pulled sufficiently~\cite{cnn, letterYoutube}, which led to a global \emph{information chaos}, 
and fueled science distrust.
Unfortunately, as LAIMs trained on unsafe data are given a more and more central role to make such trolley problem decisions, 
there is a serious risk that disinformation campaigns may become increasingly empowered.

\section{Alchemical fixes}\label{sec:false_good_idea}
In a highly commented talk for the 2018 conference on Neural Information Processing (NeurIPS), Ali Rahimi compared modern machine learning to \emph{alchemy}~\cite{hutson2018}.
It ``worked'', but ``alchemists also believed they could cure diseases with leeches and transmute base metals into gold''.
Unfortunately, currently, as opposed to aiming for a deeper understanding of the failure modes of machine learning, 
many developers of LAIMs instead favor more ``alchemical fixes'', despite a lack of security guarantees and theoretical justifications.
In this section, we argue that such alchemical fixes are unlikely to provide lasting solutions to the security and privacy issues of LAIMs.

\subsection{Troubleshooting}
Today's main solution to validate the security of LAIMs is empirical testing, without complementing it with provable guarantees.
Unfortunately, there is currently a lack of automated solutions to detect systematic bias, misinformation, and privacy leaks of LAIMs.
As a result, most of the troubleshooting has relied on human reviewing, 
and has often followed the large-scale deployment of the LAIM~\cite{AbidFZ21,CarliniTWJH+20,McGuffieNewhouse20}.
Radically larger investments seem urgent to stress-test such dangerous algorithms.

Having said this, even with large investments, human oversight arguably does not match the scales of LAIMs,
as the set of possible prompts to LAIMs is combinatorially large, while actual user queries are also very heterogenous.
Indeed, every day, 15\% of Google's search queries have never been made before~\cite{google_search_blog}.
As a result, most of users' (future) queries cannot be tested or checked by human oversight alone.
In fact, even automated testing can only verify a tiny fraction of the exponential number of sensitive prompts.

As an example, \cite{allgaier2019} showed that, while YouTube searches on ``Climate Change'' or ``Global Warming'' return scientific responses, the results for ``Climate Manipulation'' or ``Climate Modification'' are widely unscientific.
YouTube recommendations are highly customized, and using LAIMs to power them is likely to worsen the trend~\cite{McGuffieNewhouse20}.
As a result, an auditor testing YouTube's climate change recommendations might erroneously conclude that YouTube only provides scientific results to its two-billion users.
Similar criticisms on the limits of manual troubleshooting have been made about other platforms. 
For instance, while TikTok removed content with the hashtag \emph{\#StoptheSteal}, linked to the Capitol attack and the coup attempt after Trump's 2020 elections defeat, it was shown to fail to ban \emph{\#StoptheStealing}~\cite{Perez20}.

Troubleshooting may also fail to detect biases against demographic populations who are underrepresented in the organization developing the algorithms~\cite{BuolamwiniG18}, 
or whose life may be undervalued by the media of the countries hosting such organizations~\cite{Wong21}. 
When queried about ongoing human rights abuse, wars and genocides in other regions of the world, all platforms offer a large panel of content promoting war, smearing or threatening human rights activists or worse, allowing abusers and banning victims from the platform. 
The double-standard in content moderation~\cite{york2021silicon, roberts2018digital} is worsened by the imbalance of fake accounts between victims and abusers, who tend to use state-scale resources to amplify their presence. 
In particular, the hope to fix LAIMs with (fake?) user feedback after deployment is a very dangerous illusion.

\subsection{Portability of fixes}
In the past couple of years, issues in already deployed LAIMs triggered series of media coverage for the companies that deployed them. In a few notable cases, the {\it observed} issue tends to be solved after the coverage, like in 2018 with non-gendered pronouns in Turkish translations~\cite{turkish2018}.
But manual fixes cannot fix an exponentially large subset of contexts that LAIMs are asked to address.
Moreover, they must be systematically adapted to new models. One more promising path is the use of automated rewriting, as was proposed and implemented in 2020~\cite{scalablefixes2020}.
However, scaling fixes remains hard.
Besides, problems that were previously fixed can reappear in updated LAIMs, as was the case in 2021 with the aforementioned issue of gender-neutrality, this time for the Hungarian language~\cite{hungarian2021}. 
At the very least, today's fixes are not reliable and/or scalable to make ever LAIMs secure.

\subsection{Fine tuning}
Fine-tuning LAIMs to smaller but more reliable datasets has been shown to improve models' performances~\cite{PopovKPVN18,ChenHCCLY20,GunelDCS21}.
Several authors~\cite{ZieglerSWBR+19,SolaimanDennison21,JinBKDNR21,LiTLH21,YuNBGI+21} have proposed to leverage fine tuning to make LAIMs more reliable, 
e.g., to prevent them from generating hate speech or to be private with respect to the fine-tuning data.
This research direction seems to reduce the harm of today's LAIMs.
However, it should be stressed that as of today, fine tuning provides little guarantee.
In fact, the example of~\cite{McGuffieNewhouse20} shows how unpredictable LAIMs can be, and suggests that algorithms may behave well in most settings and can become major disinformation engines when prompted in unexpected ways.
Arguably, thus far, we do not yet have a sufficient understanding and control over the latter in order to confidently deploy large models at scale.

\subsection{Pre-prompting}
ChatGPT and New Bing have been shown to use \emph{pre-prompting} to prevent them from leaking sensitive information or generating dangerous outputs.
Typically, the LAIM is first given a description of a ``good'' AI in natural language (e.g. 
``if the user requests content that is harmful to someone ... then [the good AI] explains and performs a very similar but harmless task''~\cite{pre-prompt})
and is then tasked to act like the described AI (still in natural language).
These instructions that precede the user's prompts are known as \emph{pre-prompts},
and they typically associate a ``good'' behavior with concealing sensitive information
and avoiding controversial topics.
However, clearly, this design principle for highly impactful algorithms is poor engineering,
and offers no security guarantee.
In fact, it has been shown that very basic so-called ``jailbreaks'', 
e.g. asking the chatbot to disregard its pre-prompt,
can successfully bypass pre-prompting measures~\cite{li2023multi},
and make LAIMs expose sensitive information of their training data.

An additional, practical limitation of current LAIM architectures is the size of their context.
Above a model-specific number of \emph{tokens}, further generation would gradually ``forget'' the pre-prompting; no matter the semantic of the user input.
While restricting the total number of tokens may fit some applications (e.g. short customer question answering), this inherent limitation may easily be overlooked in actual deployments (especially as less skilled practitioners get access to such LAIM models), negating the effect of pre-prompting altogether.

\subsection{Teaching what is sensitive}
One seemingly promising approach consists of teaching algorithms what messages are desirable or undesirable to produce.
This solution is often known as algorithmic \emph{alignment}~\cite{Soares15, hoang2019fabuleux}.
Essentially, it aims to make algorithms' objective functions aligned with human preferences; 
or rather, to align them with the result of a vote between humans~\cite{NoothigattuGADR18,LeeKKKYCSNLPP19,tournesol}.
Such an hypothetical \emph{aligned} algorithm could learn what kind of messages violate user privacy, label training texts as ``sensitive'' or ``non-sensitive'', and thereby output a \emph{cleaned} non-sensitive training database.
This approach, essentially proposed by~\cite{ShiCLJY22}, might even address the privacy ambiguity discussed in Section~\ref{sec:privacy_ambiguity}.
However, there is currently no reliable and robust solution to the alignment problem, 
and a strong theory of robust alignment for LAIMs is arguably lacking.
In fact, what may be most lacking today is a large-scale \emph{secure} database of reliable human judgments to solve alignment~\cite{tournesol}.

\section{Conclusion}\label{sec:conclusion}
This paper emphasized three characteristics of the data on which LAIMs are trained.
Namely, they are mostly \emph{user-generated}, \emph{very high-dimensional} and \emph{heterogeneous}.
Unfortunately, the current literature on secure learning, which we reviewed, shows that these features make LAIMs inherently vulnerable to privacy and poisoning attacks.
\emph{Large AI models are bound to be dangerous}.
Their rushed deployment, especially at scale, poses a serious threat to justice, public health and to national and international security.

\subsection{Future work}
To build genuinely secure AIs, it is urgent that the scientific community genuinely prioritize 
some research directions over the blind quest of benchmark performances, or of theorems under unrealistic (e.g. i.i.d.) assumptions with questionable political motivations (e.g. fitting to all data).
Below we list three research directions which, we believe, should be given a lot more attention.

\paragraph{Correlated differential privacy.}
As explained in Section~\ref{sec:privacy_ambiguity},
differential privacy is failing to account for privacy leaks through other users' datasets.
This huge flaw of today's leading privacy concept must urgently be addressed, e.g. by \emph{correlated differential privacy}~\cite{KiferM11,ZhuXLZ15}.
Designing training schemes that provide such stronger privacy guarantee is one of the great upcoming challenges
for AI researchers,
in order to combine the promises of machine learning with what international law regards as a human rights.

\paragraph{Certifying data providers.}
To combat poisoners, especially in the context of increased fabricated online activities by powerful actors,
it seems urgent to provide much more reliable tools to authenticate and certify data providers.
In particular, a gold standard would be to guarantee Proof of Personhood (PoP)~\cite{DBLP:conf/eurosp/BorgeKJGGF17,DBLP:journals/corr/abs-2011-02412},
i.e. assigning to each human being a unique digital verifiable identifier.
Several approaches, typically based on a \emph{web of trust}, aim to provide approximate PoPs~\cite{KamvarSG03,DanezisM09,LafourcadeLombard19,MaramMZJFKLMJM20,PoupkoSST21,DBLP:journals/corr/abs-2211-01179}.
Similar techniques may also be useful to certify data providers' expertise and legitimacy.
Finally, secure learning algorithms should be designed to leverage such data,
perhaps as was done by~\cite{nitzan1982}.

\paragraph{Building large secure public datasets of human judgments.}
Research in machine learning strongly relies on datasets to test models.
However, so far, the most widely used datasets are either of low social value 
(e.g. recognizing figures in images)
or are highly unsafe (e.g. crawled web data).
A lot more efforts must arguably be made to build large secure public datasets on what matters most for social development,
like e.g. human judgments on how impactful AIs must behave.
A few initiatives already exist in this regard~\cite{awad2018moral,hoang2021tournesol},
and they should arguably be given more attention.

\subsection{Calls to action}
Given our survey, we make three calls to different communities who, we believe, have a key role to play 
to protect our societies against out-of-control insecure information technologies.

\paragraph{To regulators.} We first call regulators to apply the principle of \emph{presumption of non-compliance}.
In light of our impossibility theorems, as well as of the numerous issues that \emph{all} LAIMs have been found to feature,
we argue that, like in essentially all mature industries (aircraft, pharmaceutical, food, automobile...), 
by default, LAIMs must be considered to be non-compliant, 
e.g. with the \emph{General Data Protection Regulation (GDPR)} or with non-discriminatory laws,
\emph{even when we fail to provide evidence for this law violation} 
(which is often made harder by companies' increased opaqueness).
In order to obtain the right to be commercialized, 
we believe that LAIMs must undergo a certification process,
which involves powerful, well-funded and independent regulatory agencies.
Put differently, the \emph{burden of proof} of compliance must fall on the developers, 
not on civil society, which too often lacks funds, time, expertise and/or data to prove non-compliance.

\paragraph{To scientists and journalists.} 
We next call scientists and journalists to urgently adopt increased levels of rigor,
especially when assessing positive claims of safety and privacy.
The current (financial) incentives to rush privacy-violating LAIM deployments are huge.
In particular, we urge them to pay attention to \emph{conflicts of interest},
which have been shown to be alarmingly huge, especially in AI, and even more so in AI ethics~\cite{abdalla2021grey}. 
Private groups have been shown to explicitly demand that their researchers ``strike a positive tone\footnote{In fact, because of one of the authors' co-affiliation, this very paper has long been stalled by Google's approval system.}''~\cite{positivetone},
in a manner unfortunately reminiscent of previous scientific disinformation campaigns led by, 
e.g., the tobacco, sugar and oil industries~\cite{oreskes2010defeating,oreskes2011}.
Such campaigns were also found to congratulate and fund scholars 
who speak positively of dangerous products,
and to degrade those who expose dangers and call for regulations.
The AI community must urgently question their involvements and dependencies on companies
and governments that are known to leverage AIs for human rights abuses.
Our scientific integrity is jeopardized by the perverse incentives that this implies.
Additionally, we ask scientists to favor the research on security 
when reviewing academic research, inviting scholars to present their work, 
recruiting researchers, promoting their colleagues and assessing grant proposals.
The current academic focus on algorithmic performance, and its inattention to social impacts, are endangering our societies.

\paragraph{To developers.} Finally, we call for a moratorium on the large-scale deployment and commercialization of large AI models in both public and private sectors, 
as well as any high-dimensional learning model that is mostly trained on \emph{user-generated}, \emph{high-dimensional}, and \emph{heterogeneous} data.
At the very least, the wide use of such \emph{dangerous} technologies should be deeply frowned upon, 
especially when it is done in a rushed manner,
as is currently too often the case. 
We especially invite the computer science community to take inspiration from the lessons learned in fields such as biology, medicine or the research on consequential public interest questions such  as tobacco control~\cite{lewis2006dealing, proctor2013ban, kagan1993banning, saffer2000effect}, including normalizing calls for bans and restrictions of deployment in scientific publications~\cite{grill2016case, proctor2013ban} when scientific arguments such as the ones we provide justify such a call. We hope that by doing so, similar mistakes are not repeated, given that similar causes are behind delaying proper measures of public interest~\cite{abdalla2021grey}.

\bibliographystyle{alpha}
\bibliography{references}

\end{document}